%% file: EW.tex
\documentclass[10pt,twocolumn,letterpaper]{article}

\usepackage[pagenumbers]{cvpr}
\input{preamble}

\usepackage[hyphens]{url}
\usepackage{graphicx}
\usepackage{amsmath,amssymb}
\usepackage{amsthm}
\newtheorem{proposition}{Proposition}
\newtheorem{remark}{Remark}
\usepackage{booktabs}
\usepackage{multirow}
\usepackage{array}
\usepackage{algorithm}
\usepackage{algpseudocode}
\usepackage{placeins}
\usepackage{dblfloatfix}
\usepackage{afterpage}

\newcommand{\best}[1]{\textbf{\mbox{#1}}}
\newcommand{\down}{$\downarrow$}
\newcommand{\up}{$\uparrow$}
\newcommand{\ours}{PredErase}

\newcommand{\ewlab}{\sffamily\small}

\definecolor{cvprblue}{rgb}{0.21,0.49,0.74}
\usepackage[pagebackref,breaklinks,colorlinks,allcolors=cvprblue]{hyperref}

\def\paperID{*****}
\def\confName{CVPR}
\def\confYear{2027}

\title{PredErase: Training-Free Object-and-Effect Removal\\
with Predictive Latent Guidance}

\author{
Waikit Xiu$^{1}$ \quad
Qiang Lu$^{2}$ \quad
Junbiao Chen$^{2}$ \quad
Xiying Li$^{2,*}$\\
$^{1}$The University of Hong Kong \quad
$^{2}$Sun Yat-sen University
}

\begin{document}
\twocolumn[{%
\renewcommand\twocolumn[1][]{#1}%
\maketitle
\vspace{-24pt}
\begin{center}
    \setlength{\tabcolsep}{0pt}%
    {\ewlab
    \begin{tabular}{@{}*{6}{>{\centering\arraybackslash}p{\dimexpr\textwidth/6\relax}@{}}}
    Input & After Remove & Input & After Remove & Input & After Remove \\[-0.15em]
    \multicolumn{6}{@{}c@{}}{\includegraphics[width=\textwidth]{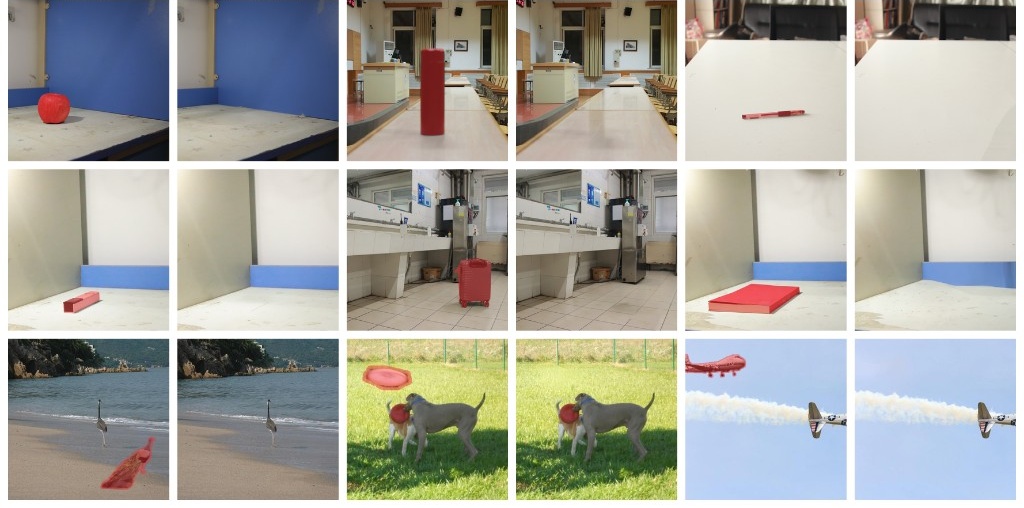}}
    \end{tabular}}
    \captionof{figure}{Object-and-effect removal with user-provided instance masks shown in red. \ours{} removes each masked object together with its cast shadows and contact shading, without task-specific training.}
    \label{fig:teaser}
\end{center}
\vspace{6pt}
}]
\begingroup
\renewcommand{\thefootnote}{\fnsymbol{footnote}}%
\footnotetext[1]{Corresponding author.}
\renewcommand{\thefootnote}{}%
\footnotetext{Code: \url{https://github.com/xiuwk0820-collab/PredErase}}
\endgroup

\begin{abstract}
Removing an object is not the same as filling its mask. Cast shadows and contact shading usually lie outside the user-provided instance mask $M_{\mathrm{obj}}$, so a frozen Fill model that edits only that mask leaves the object's photometric footprint on nearby surfaces. Supervised removers learn this joint erasure from paired clean plates. Training-free editors freeze pretrained weights, yet most still treat $M_{\mathrm{obj}}$ as the entire editable support and steer sampling with CLIP or DINO energies that do not predict the occluded scene.
We present \ours{}, a training-free inference procedure on frozen FLUX.2 and I-JEPA. The method separates where Fill may rewrite pixels from what structure should occupy the hole. A contact-band expansion $M_{\mathrm{flux}}\supseteq M_{\mathrm{obj}}$ exposes local residuals on the supporting plane. I-JEPA, pretrained for masked token prediction, supplies a context-conditioned hole target in representation space; sparse projected gradients align decoded Fill completions with that target inside the instance, while coordinates outside the packed support stay locked. Under instance-only masks on RemovalBench, RORD-Val, and DEFACTO-Val, \ours{} improves the native FLUX.2 backbone. Supervised removers remain stronger on several full-image appearance metrics; the supported claim is training-free object-and-effect editing of frozen Fill, not replacement of paired-data erasers.
\end{abstract}

% \vspace{-40pt}
\section{Introduction}
%Object removal seeks an edited image in which an unwanted instance appears never to have existed~\cite{pathak2016context,yu2018context,hu2019gatedconv,cao2023zits,suvorov2022lama,lugmayr2022repaint,rombach2022ldm,flux2024}.
Object removal is a fundamental task in image editing, with broad applications in photo retouching, content creation, and scene manipulation~\cite{pathak2016context,yu2018context,hu2019gatedconv,cao2023zits,suvorov2022lama}. 
Recent advances in diffusion and flow-matching generative models have substantially improved image inpainting, allowing user-specified object instances to be removed from complex scenes with increasingly plausible visual results~\cite{lugmayr2022repaint,rombach2022ldm,flux2024}. 
%Yet an object-only user mask $M_{\mathrm{obj}}$ rarely captures the full visual footprint of a grounded instance: cast shadows and contact shading often extend onto the surrounding support~\cite{wei2025omnieraser,wang2020instanceshadow,sanin2012survey,liu2021shadow,guo2023shadowformer,barron2012intrinsic,vicente2014shadow}.
However, removing an object is not equivalent to removing its presence. Objects interact with nearby surfaces and leave photometric traces, including cast shadows and contact shading~\cite{wei2025omnieraser,wang2020instanceshadow,liu2021shadow,guo2023shadowformer,barron2012intrinsic,sanin2012survey}.
We study \emph{object-and-effect removal} under instance-only masks: erase the specified object together with local, geometry-coupled residuals on adjacent support, while leaving unrelated content unchanged. Mirror reflections, refraction, and global illumination are outside this setting.

A frozen Fill model restricted to $M_{\mathrm{obj}}$ may complete the masked region while leaving those traces unchanged (Fig.~\ref{fig:teaser}). Scores computed only inside $M_{\mathrm{obj}}$ miss this failure. RemovalBench and RORD-Val therefore measure full-image agreement with clean-plate ground truth under the OmniEraser protocol~\cite{wei2025omnieraser} (Technical Appendix: Metrics).

%%%
%A frozen Fill model restricted to $M_{\mathrm{obj}}$ may plausibly complete the hole while preserving these outside-mask traces (Fig.~\ref{fig:teaser}).
%Because mask-restricted scores can miss this failure~\cite{zhang2018lpips,jayasumana2024cmmd,heusel2017fid,wang2004ssim}, RemovalBench and RORD-Val measure full-image agreement with clean-plate ground truth under the OmniEraser protocol~\cite{wei2025omnieraser}.
%We target this benchmark regime: grounded instances whose dominant external effects are cast shadows and contact shading on a nearby support, rather than reflections, floating objects, or effects that require a different physical model (supplementary material).
%%%%
Within this scope, supervised effect-aware removers close the gap with object--clean-plate pairs~\cite{ju2024brushnet,wei2025omnieraser,jiang2025smarteraser,zhao2026objectclear,wang2025metashadow,ke2022harmonizer}, typically at the cost of paired supervision and backbone-specific adaptation~\cite{hu2022lora}. Training-free editors instead freeze pretrained weights and steer sampling~\cite{meng2022sdedit,chung2022dps,couairon2023diffedit,yu2023freedom,epstein2023selfguidance,sun2025attentive,ekin2024clipaway}. On object removal they almost always keep $M_{\mathrm{obj}}$ as the editable support: self-attention surgery or CLIP energies suppress the instance \emph{inside} the mask, while shadows and reflections encoded in unmasked tokens are left untouched unless the user dilates the mask by hand. Semantic objectives (CLIP/DINO) further do not predict the scene hidden by the object.

We therefore treat training-free object-and-effect removal as constrained trajectory steering of a frozen Fill model~\cite{dhariwal2021diffusion,karras2022edm,lipman2023flow}, factored into \emph{where} the sampler may rewrite pixels and \emph{what} should occupy the hole (Fig.~\ref{fig:pipeline}).
Visible residuals that still depend on the instance after it is hidden must enter the editable support; we instantiate that support with a contact-band expansion $M_{\mathrm{flux}}\supseteq M_{\mathrm{obj}}$ for upright, ground-contact shadows.
For the hole, frozen I-JEPA~\cite{assran2023ijepa,lecun2022path} provides a predictive representation target: it is pretrained to infer masked tokens from context. Alignment back-propagates through the decoder into Fill latents in the manner of measurement guidance~\cite{chung2022dps}. Projection onto the packed support $P$ leaves visible coordinates unchanged.

We evaluate \ours{} on RemovalBench and RORD-Val under the OmniEraser protocol~\cite{wei2025omnieraser} and on DEFACTO-Val under the SmartEraser protocol~\cite{jiang2025smarteraser}, using frozen FLUX.2-klein-4B~\cite{flux2024,esser2024sd3}.
Relative to native FLUX.2, \ours{} improves CMMD and PSNR on RemovalBench and FID, CMMD, LPIPS, and PSNR on RORD-Val (Tables~\ref{tab:main}--\ref{tab:defacto}).
On RemovalBench, supervised OmniEraser retains lower FID and LPIPS.

Our contributions are:
\begin{itemize}
\item A training-free object-and-effect procedure that edits frozen Fill latents under instance-only masks, with no paired clean plates and no weight updates.
\item A factorization of editable support $M_{\mathrm{flux}}$ from I-JEPA hole prediction, with projected decoder-side updates confined to the packed latent mask.
\item Protocol-aligned comparisons against frozen Fill and training-free removers, with supervised erasers reported on the same splits, plus component and prior-swap ablations.
\end{itemize}

%Our contributions are summarized as follows:
%\begin{itemize}
%    \item A training-free formulation of \emph{instance-only} object-and-effect removal that couples frozen Fill with predictive world-model guidance, requiring neither paired removal data nor parameter updates.
%    \item A factorized inference design that separates effect-aware edit support ($M_{\mathrm{flux}}$) from context-conditioned hole prediction (frozen I-JEPA), with projected latent updates that preserve visible content.
%    \item Protocol-aligned evaluation on RemovalBench, RORD-Val, and DEFACTO-Val, including component ablations and matched I-JEPA/CLIP/DINOv2 critic swaps; primary claims use FLUX.2-klein-4B.
%\end{itemize}

\section{Related Work}

\paragraph{Inpainting and supervised removers.}
Context encoders and diffusion Fill models~\cite{pathak2016context,rombach2022ldm,flux2024} complete content inside an object mask, yet cast shadows and contact shading often remain outside that mask.
RemovalBench~\cite{wei2025omnieraser} and DEFACTO-Val~\cite{jiang2025smarteraser} make this failure mode measurable under instance-only masks.
OmniEraser~\cite{wei2025omnieraser}, SmartEraser~\cite{jiang2025smarteraser}, ObjectClear~\cite{zhao2026objectclear}, and MetaShadow~\cite{wang2025metashadow} learn effect-aware removal from paired clean plates, often with LoRA~\cite{hu2022lora}.
Shadow restorers such as ShadowFormer~\cite{guo2023shadowformer} instead assume shadow masks or shadow-centric data.
Across these lines, either the edit stays inside the user mask, or training needs paired supervision beyond the instance-only setting used here.

\paragraph{Training-free editing.}
Inference-time editors update frozen diffusion or flow latents without weight training~\cite{meng2022sdedit,chung2022dps}.
Attentive Eraser~\cite{sun2025attentive} and CLIPAway~\cite{ekin2024clipaway} are representative object-removal variants; related controls include DiffEdit~\cite{couairon2023diffedit} and FreeDoM~\cite{yu2023freedom}.
Most keep the object mask as the editable support and drive updates with CLIP~\cite{radford2021clip} or DINO~\cite{oquab2024dinov2} energies.
Shadows and reflections persist unless they already lie inside the mask, because they are encoded in unmasked tokens.

\paragraph{Predictive representations as hole priors.}
Masked predictive learning treats missing content as something to infer from visible context~\cite{lecun2022path,he2022mae,assran2023ijepa,bardes2024vjepa}.
I-JEPA predicts tokens for masked regions without decoding pixels; unlike CLIP or DINOv2, the pretraining question is exactly \emph{what occupies the hole}.
We use the frozen predictor at inference as a hole prior in token space.

\begin{figure*}[!t]
    \centering
    \includegraphics[width=\textwidth]{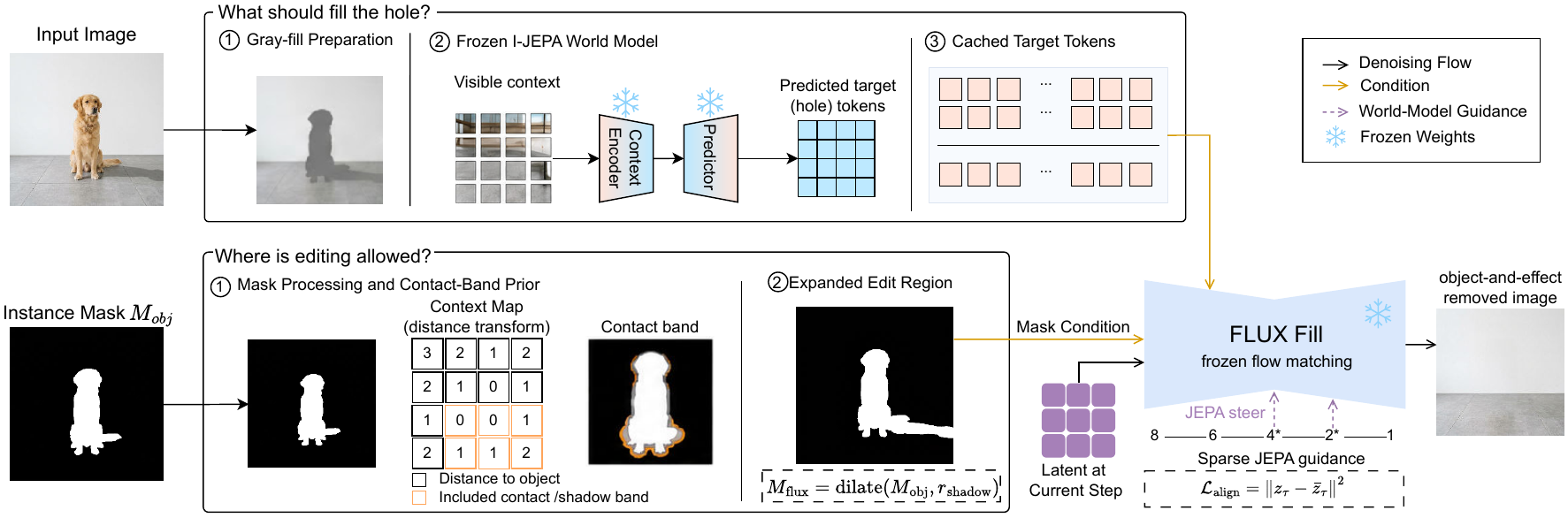}
    \caption{\ours{} training-free pipeline.
    \emph{Top:} gray-filled $I_{\mathrm{vis}}$ lets frozen I-JEPA predict and cache hole tokens $\mathbf{E}_{\mathrm{target}}$.
    \emph{Bottom:} a contact-band prior expands $M_{\mathrm{obj}}$ into the effect-aware Fill support $M_{\mathrm{flux}}$.
    \emph{Center:} frozen FLUX.2-klein-4B follows its native trajectory; at $t\in\{4,2\}$, decoded features are aligned with $\mathbf{E}_{\mathrm{target}}$ (Eq.~\ref{eq:align}) and projected so that the additional I-JEPA update is zero outside the editable latent support (Eq.~\ref{eq:pin}). No model weights are updated.}
    \label{fig:pipeline}
\end{figure*}

\section{Methodology}
\label{sec:method}

\subsection{Problem Setup}
\label{sec:precompute}

Given an RGB image $I \in \mathbb{R}^{3 \times H \times W}$ and a binary object mask $M_{\mathrm{obj}} \in \{0,1\}^{H \times W}$ (1 denotes removal), we seek an output $I_{\mathrm{out}}$ in which both the masked instance and its associated outside-mask effects, including cast shadows and contact shading, are absent.

For an editable support $M$, a frozen flow-matching Fill model induces the conditional law shown in Eq.~\eqref{eq:fill-law}:
\begin{equation}
    I' \sim p_{\mathrm{Fill}}\!\left(\,\cdot\mid I\odot(1-M),M\right).
    \label{eq:fill-law}
\end{equation}
With $M=M_{\mathrm{obj}}$, residuals outside the mask remain fixed as conditioning evidence, while the hole is drawn from a local completion prior that need not match the broader scene.
Training-free object-and-effect removal therefore faces two coupled issues: local Fill statistics underdetermine the revealed content, and mask-local sampling cannot edit residuals outside $M_{\mathrm{obj}}$.

\subsection{\ours{}}
\label{sec:overview}

\ours{} is a training-free inference method that freezes a Fill generator and an I-JEPA predictor, and only redirects the Fill sampling trajectory at test time (Fig.~\ref{fig:pipeline}).
I-JEPA is a hole prior in \emph{representation} space; Fill remains responsible for appearance, including residual illumination inside $M_{\mathrm{flux}}$.

\paragraph{Precomputation.}
The object mask is used in two ways.
On the predictive branch we gray-fill the instance as in Eq.~\eqref{eq:grayfill},
\begin{equation}
    I_{\mathrm{vis}} = I \odot \bigl(1-M_{\mathrm{obj}}\bigr) + \mathbf{g} \odot M_{\mathrm{obj}},
    \label{eq:grayfill}
\end{equation}
where $\mathbf{g}=(0.5,0.5,0.5)$ is fixed gray RGB,
and cache the frozen I-JEPA target $\mathbf{E}_{\mathrm{target}}$ once from this visible stream.
On the generative branch we expand $M_{\mathrm{obj}}$ into the effect-aware support $M_{\mathrm{flux}}$ (Eq.~\eqref{eq:mflux}).
The first branch sets what should replace the object; the second sets where Fill may edit.

\paragraph{Sampling.}
We \emph{source-prefill} by encoding $I\odot(1-M_{\mathrm{flux}})$, then run conditioned Fill steps (mask, shadow-aware prompt, CFG; see Technical Appendix: Implementation details).
$\mathbf{E}_{\mathrm{target}}$ is cached once; sampler seeds provide stochasticity.
At $t\in\mathcal{T}_{\mathrm{guide}}$, the decoded completion is aligned with $\mathbf{E}_{\mathrm{target}}$ via Eq.~\eqref{eq:align} and projected onto the editable latent subspace by Eqs.~\eqref{eq:grad}--\eqref{eq:pin}; other steps stay unchanged.
I-JEPA shapes hole structure inside $M_{\mathrm{obj}}$, while $M_{\mathrm{flux}}$ and the prompt expose residuals for Fill; the projection blocks guidance outside $P$.
Alg.~\ref{alg:erase} summarizes the routine; operators follow below.

\begin{algorithm}[!t]
    \caption{\ours{} inference. $\mathrm{FMStep}$ denotes one conditioned Fill transition (mask, prompt, CFG); $\mathrm{UpdateAndLock}$ implements Eqs.~\ref{eq:grad}--\ref{eq:pin}.}
    \label{alg:erase}
    \renewcommand{\algorithmicrequire}{\textbf{Input:}}
    \renewcommand{\algorithmicensure}{\textbf{Output:}}
    \begin{algorithmic}[1]
    \Require $I$, $M_{\mathrm{obj}}$
    \Ensure $I_{\mathrm{out}}$
    \State $I_{\mathrm{vis}} \leftarrow \mathrm{GrayFill}(I, M_{\mathrm{obj}})$
    \State Cache $\mathbf{E}_{\mathrm{target}}$ from frozen I-JEPA on $I_{\mathrm{vis}}$ (once)
    \State Build $M_{\mathrm{shadow}}$; $M_{\mathrm{flux}} \leftarrow \mathrm{dilate}_{4}(M_{\mathrm{obj}} \cup M_{\mathrm{shadow}})$
    \State Pack editable latent mask $P$ from $M_{\mathrm{flux}}$
    \State $\mathbf{z}_{T} \leftarrow \mathrm{Encode}(I \odot (1-M_{\mathrm{flux}}))$ \Comment{source prefill}
    \For{$t = T$ \textbf{down to} $1$}
        \State $\mathbf{z}_{t-1} \leftarrow \mathrm{FMStep}(\mathbf{z}_{t}, t;\, M_{\mathrm{flux}}, \mathrm{prompt})$
        \If{$t \in \mathcal{T}_{\mathrm{guide}}$}
            \State $\mathbf{z}^{\mathrm{pin}} \leftarrow \mathbf{z}_{t-1}$; decode preview $\hat{I}$
            \State Compute $\mathcal{L}_{\mathrm{align}}$ and $G \leftarrow \nabla_{\mathbf{z}}\mathcal{L}_{\mathrm{align}}(\mathbf{z}^{\mathrm{pin}})$
            \State $\mathbf{z}_{t-1} \leftarrow \mathrm{UpdateAndLock}(\mathbf{z}^{\mathrm{pin}}, G, P)$
        \EndIf
    \EndFor
    \State $I_{\mathrm{out}} \leftarrow \mathrm{Decode}(\mathbf{z}_{0})$
    \end{algorithmic}
\end{algorithm}

\subsection{Predictive Representation Prior}
\label{sec:guidance}

The main difficulty in training-free removal is not local texture synthesis, but what structure should occupy the hole once the instance is gone.
I-JEPA is pretrained to predict masked patch representations from visible context, which is the same interface as object removal.
We use the public ViT-H/14-1K stack as a frozen pair $(\phi_{\mathrm{JEPA}},\psi_{\mathrm{JEPA}})$ and cache $\mathbf{E}_{\mathrm{target}}$ as a context-conditioned hole target in patch space.
With $I_{\mathrm{vis}}$ gray-filled on the instance and $(\mathcal{I}_{\mathrm{vis}},\mathcal{I}_{\mathrm{mask}})$ the visible/object patch sets, the target is given by Eq.~\eqref{eq:jepa-target},
\begin{equation}
    \mathbf{E}_{\mathrm{target}} = \psi_{\mathrm{JEPA}}\!\bigl(\phi_{\mathrm{JEPA}}(I_{\mathrm{vis}}),\, \mathcal{I}_{\mathrm{vis}},\, \mathcal{I}_{\mathrm{mask}}\bigr)
    \label{eq:jepa-target}
\end{equation}
and is held fixed during sampling: $\phi_{\mathrm{JEPA}}(I_{\mathrm{vis}})$ supplies tokens on $\mathcal{I}_{\mathrm{vis}}$, from which $\psi_{\mathrm{JEPA}}$ predicts $\mathbf{E}_{\mathrm{target}}$ on $\mathcal{I}_{\mathrm{mask}}$.
Gray-filling $M_{\mathrm{obj}}$ (Eq.~\eqref{eq:grayfill}) keeps object appearance out of the visible stream; excluding $\mathcal{I}_{\mathrm{mask}}$ from the visible set preserves I-JEPA's masked-prediction interface.

\paragraph{Effect-aware edit region.}
\label{sec:shadow}
Cast shadows and contact shading on the supporting plane typically extend beyond $M_{\mathrm{obj}}$.
Fill can erase those residuals only if they lie in the editable support.
Without estimating a light source or BRDF, we construct a contact band as a geometric surrogate of that dependent set for upright, ground-contacted objects, the dominant regime of RemovalBench.
Fig.~\ref{fig:pipeline} (bottom branch) extracts the floor-contact segment $\mathcal{C}$ (the lowest rows of $M_{\mathrm{obj}}$ under an upright-image convention) and forms the one-sided band in Eq.~\eqref{eq:shadow-band},
\begin{equation}
    M_{\mathrm{shadow}}=\{c+an+b\tau\mid c\in\mathcal{C},\,a\in[0,\sigma],\,|b|\leq\delta_x\}\cap\Omega,
    \label{eq:shadow-band}
\end{equation}
with scale-adaptive extents as in Eq.~\eqref{eq:wedge},
\begin{equation}
    (\sigma,\delta_x)=\left(\max(6,0.5h_{\mathrm{obj}}),\,\max(8,0.35w_{\mathrm{obj}})\right),
    \label{eq:wedge}
\end{equation}
and editable support as in Eq.~\eqref{eq:mflux},
\begin{equation}
    M_{\mathrm{flux}}
    =
    \operatorname{Dilate}\!\left(
        M_{\mathrm{obj}} \cup M_{\mathrm{shadow}};\,
        r=4
    \right).
    \label{eq:mflux}
\end{equation}
Here $\Omega$ is the discrete image grid, $n$ the image-plane contact normal, and $\tau$ its tangent.
For upright inputs we use $n=(0,1)$ and $\tau=(1,0)$; non-upright inputs need an externally supplied orientation, since we do not estimate support geometry.
Guided Fill can therefore rewrite the instance and its residual region jointly.
Relative coefficients are fixed across experiments; absolute extents adapt to $(h_{\mathrm{obj}},w_{\mathrm{obj}})$ (see Technical Appendix Table~1).
$I_{\mathrm{vis}}$ still gray-fills only $M_{\mathrm{obj}}$ (Eq.~\eqref{eq:grayfill}), so I-JEPA sees a clean masked-prediction input while Fill can revise contact residuals under $M_{\mathrm{flux}}$.
When orientation metadata are available, the same band is built after aligning the contact normal.
Additional operator details appear in Technical Appendix: Implementation details.

\paragraph{Completion reweighting.}
Native Fill is optimized for local completion and can look textured while still breaking surface continuation, leaving object identity, or mismatching illumination.
We keep $\mathbf{E}_{\mathrm{target}}$ (Eq.~\eqref{eq:jepa-target}) fixed and steer the Fill trajectory toward it with sparse projected-gradient steps (Eqs.~\eqref{eq:grad}--\eqref{eq:pin}).
A matched comparison of I-JEPA against CLIP and DINOv2 hole energies is in Table~\ref{tab:ablation-prior} and Technical Appendix: Guidance prior comparison.
At a guided step we decode a preview $\hat{I}_t$, embed it with the same frozen $\phi_{\mathrm{JEPA}}$, and compare object-patch tokens as in Eq.~\eqref{eq:token-pair},
\begin{equation}
    \mathbf{u}_i=\phi_{\mathrm{JEPA}}(\hat{I}_t)_i,\qquad \mathbf{v}_i=\mathbf{E}_{\mathrm{target},i},
    \label{eq:token-pair}
\end{equation}
yielding the alignment loss in Eq.~\eqref{eq:align}:
\begin{equation}
    \mathcal{L}_{\mathrm{align}}=\frac{1}{|\mathcal{I}_{\mathrm{mask}}|}\sum_{i\in\mathcal{I}_{\mathrm{mask}}}\left\|\mathbf{u}_i-\mathbf{v}_i\right\|_2^2.
    \label{eq:align}
\end{equation}
Both sides live in the same I-JEPA patch space: $\mathbf{v}_i$ is the cached target token and $\mathbf{u}_i$ is the current completion at that index.
Minimizing $\mathcal{L}_{\mathrm{align}}$ biases the trajectory so the decoded hole matches context-predictable structure; Fill still synthesizes appearance.
Restricting the sum in Eq.~\eqref{eq:align} to $\mathcal{I}_{\mathrm{mask}}$ leaves Fill free to clean residuals over the larger $M_{\mathrm{flux}}$: effect cleanup is driven by the expanded support and prompt, whereas I-JEPA constrains structure inside the object hole.

\paragraph{Guidance schedule.}
Guidance needs a reasonably formed decoded preview, but applying it at every step tends to over-constrain details that Fill already models well.
We therefore evaluate Eq.~\eqref{eq:align} only at $\mathcal{T}_{\mathrm{guide}}=\{4,2\}$ for $T{=}14$ (Fig.~\ref{fig:pipeline}; see Technical Appendix Table~1).
This sparse schedule adds two late corrections and leaves the remaining flow-matching steps unchanged.

\subsection{Latent Update}

At a guided transition, let $\mathbf{z}^{\mathrm{pin}}=\bar{\mathbf{z}}_{t-1}$ be the native Fill state and $P\in\{0,1\}^{d}$ the packed-latent image of $M_{\mathrm{flux}}$ (see Technical Appendix: Implementation details).
Differentiating $\mathcal{L}_{\mathrm{align}}$ in Eq.~\eqref{eq:align} through the frozen decoder and I-JEPA encoder gives $G=\nabla_{\mathbf{z}}\mathcal{L}_{\mathrm{align}}(\mathbf{z}^{\mathrm{pin}})$.
We form the unconstrained proposal in Eq.~\eqref{eq:grad},
\begin{equation}
    \tilde{\mathbf{z}}=\mathbf{z}^{\mathrm{pin}}-\eta G,
    \qquad \eta=0.45,
    \label{eq:grad}
\end{equation}
and project onto the editable subspace determined by $P$ as in Eq.~\eqref{eq:pin} (see Technical Appendix Table~1):
\begin{equation}
    \mathbf{z}^{+}
    =P\odot\tilde{\mathbf{z}}+(1-P)\odot\mathbf{z}^{\mathrm{pin}}.
    \label{eq:pin}
\end{equation}
Coordinates with $P=1$ may move; all others are restored to the current native Fill state $\mathbf{z}^{\mathrm{pin}}$ (not a frozen source encoding).
Eqs.~\eqref{eq:grad}--\eqref{eq:pin} therefore amount to a projected-gradient step on the editable subspace: guidance cannot rewrite outside-$P$ coordinates at that micro-step, while Fill remains responsible for appearance synthesis under $M_{\mathrm{flux}}$.

\begin{figure*}[!t]
    \centering
    \setlength{\tabcolsep}{0pt}%
    {\ewlab
    \begin{tabular}{@{}*{6}{>{\centering\arraybackslash}p{\dimexpr\textwidth/6\relax}@{}}}
    Input & Mask & Native FLUX.2 & OmniEraser & Ours & GT \\[-0.15em]
    \multicolumn{6}{@{}c@{}}{\includegraphics[width=\textwidth]{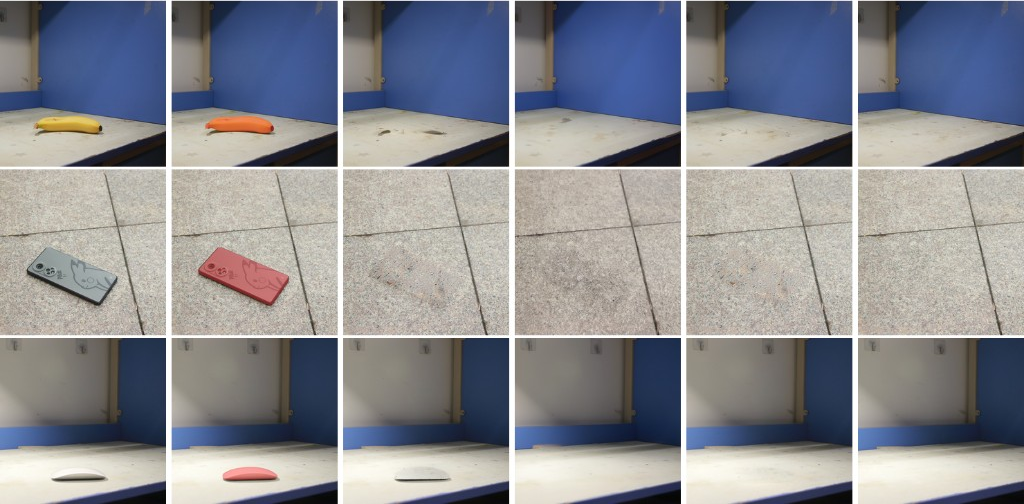}}
    \end{tabular}}
    \caption{Qualitative comparison on RemovalBench ($1024{\times}1024$, instance-only masks).
    Columns: Input, $M_{\mathrm{obj}}$, native FLUX.2-klein-4B, supervised OmniEraser, \ours{} (FLUX.2), clean-plate GT.
    Rows (top to bottom): banana on desk, phone on tile, desk mouse.}
    \label{fig:qual}
\end{figure*}%

\section{Experiments}
\label{sec:experiments}

\subsection{Experimental Setup}

\paragraph{Benchmarks.}
We evaluate on three public benchmarks (split sizes and scoring scripts are listed in Technical Appendix: Metrics):
\begin{itemize}
    \item RemovalBench~\cite{wei2025omnieraser}: $69$ real object--clean-plate pairs at $1024{\times}1024$ with instance-only masks under the OmniEraser protocol.
    Scores are full-image FID, CMMD, LPIPS (SqueezeNet), PSNR, and Aesthetic Score (AS); this split is our primary test for outside-mask cast shadows and contact shading.

    \item RORD-Val~\cite{wei2025omnieraser}: a larger clean-plate validation set evaluated with the same OmniEraser metrics and table layout as RemovalBench (FID, CMMD, LPIPS, PSNR, AS).

    \item DEFACTO-Val~\cite{jiang2025smarteraser}: synthetic removals under the SmartEraser protocol at $1024{\times}1024$, scored with CMMD, ReMOVE (SAM ViT-H), AlexNet LPIPS, SSIM, and PSNR via the public metric code.
\end{itemize}
Local \ours{} scores average five seeds (not best-of-five); see Technical Appendix: Multi-seed tests.

\paragraph{Baselines.}
We compare \ours{} against classical/diffusion inpainters (ZITS++, LaMa, BrushNet, FLUX.1-Fill), training-free editors (CLIPAway, Attentive Eraser), and supervised removers (OmniEraser, SmartEraser, PowerPaint) under each suite's official protocol~\cite{wei2025omnieraser,jiang2025smarteraser} (Tables~\ref{tab:main}--\ref{tab:defacto}); DEFACTO-Val further includes RePaint and SD-Inpaint.
All reported methods are evaluated in our environment on the official splits, masks, and ground truth, using the public metric implementations.
Native FLUX.2-klein-4B is the frozen-backbone operating point; Table~\ref{tab:ablation} reports matched controls.
Methods discussed only in Related Work (e.g., ObjectClear) are not included in these tables.

\begin{table*}[!t]
  \centering
  \footnotesize
  \setlength{\tabcolsep}{7.5pt}
  \setlength{\extrarowheight}{1.2pt}
  \renewcommand{\arraystretch}{1.08}
  \caption{Quantitative comparison on RemovalBench and RORD-Val under the OmniEraser evaluation suite~\cite{wei2025omnieraser}.
  \textbf{Bold} = best per column among listed methods; OmniEraser is a paired-training reference.
  All methods are scored in our environment with the official full-image FID/CMMD/LPIPS/PSNR/AS protocol (SqueezeNet LPIPS).
  Mask-restricted DINO is omitted from this table (Technical Appendix: Metrics).}
  \label{tab:main}
  \begin{tabular}{l *{5}{r} *{5}{r}}
  \toprule
  & \multicolumn{5}{c}{\textbf{RemovalBench}} & \multicolumn{5}{c}{\textbf{RORD-Val}} \\
  \cmidrule(lr){2-6}\cmidrule(lr){7-11}
  \textbf{Method}
    & FID\down & CMMD\down & LPIPS\down & PSNR\up & AS\up
    & FID\down & CMMD\down & LPIPS\down & PSNR\up & AS\up \\
  \midrule
  ZITS++
    & 108.38 & 0.374 & 0.158 & 19.62 & 4.56
    & 107.44 & 0.448 & 0.274 & 21.17 & 4.12 \\
  LaMa
    & 99.88 & 0.351 & 0.156 & 18.72 & 4.55
    & 100.21 & 0.294 & 0.229 & 20.50 & 4.23 \\
  BrushNet
    & 120.97 & 0.549 & 0.191 & 18.68 & 4.63
    & 234.87 & 0.745 & 0.293 & 16.51 & 4.41 \\
  FLUX.1-Fill
    & 115.79 & 0.487 & 0.193 & 17.12 & 4.59
    & 141.39 & 0.450 & 0.217 & 18.50 & 4.55 \\
  CLIPAway
    & 108.40 & 0.272 & 0.254 & 18.78 & 4.48
    & 81.28 & 0.545 & 0.278 & 16.36 & 4.19 \\
  Attentive Eraser
    & 55.49 & 0.232 & 0.146 & 20.60 & 4.50
    & 96.77 & 0.233 & 0.221 & 20.24 & 4.77 \\
  OmniEraser
    & \best{39.52} & 0.208 & \best{0.133} & 21.11 & 4.66
    & \best{43.71} & \best{0.153} & 0.166 & 22.13 & 4.99 \\
  FLUX.2 (native)
    & 113.92 & 0.496 & 0.184 & 22.70 & 4.62
    & 149.02 & 0.644 & 0.168 & 18.49 & \best{5.08} \\
  \midrule
  \ours{} (FLUX.2)
    & 52.69 & \best{0.108} & 0.175 & \best{24.36} & \best{4.74}
    & 55.59 & 0.305 & \best{0.114} & \best{23.45} & 4.84 \\
  \bottomrule
  \end{tabular}
\end{table*}

\begin{table}[!t]
  \centering
  \footnotesize
  \setlength{\tabcolsep}{3.0pt}
  \setlength{\extrarowheight}{1.2pt}
  \renewcommand{\arraystretch}{1.05}
  \caption{Quantitative comparison on DEFACTO-Val under the SmartEraser protocol~\cite{jiang2025smarteraser}
  (AlexNet LPIPS; ReMOVE with SAM ViT-H).
  \textbf{Bold} = best per column.
  All methods are scored in our environment with the public metric code.
  We report the five axes available for all listed methods under this setup.}
  \label{tab:defacto}
  \resizebox{\columnwidth}{!}{%
  \begin{tabular}{lrrrrr}
  \toprule
  \textbf{Method}
    & CMMD\down & ReMOVE\up & LPIPS\down & SSIM\up & PSNR\up \\
  \midrule
  ZITS++
    & 0.229 & 0.899 & 0.350 & 0.634 & 19.45 \\
  LaMa
    & 0.201 & 0.926 & 0.351 & 0.683 & 21.74 \\
  RePaint
    & 0.822 & 0.924 & 0.325 & 0.644 & 20.61 \\
  SD-Inpaint
    & 0.196 & 0.912 & 0.320 & 0.656 & 21.69 \\
  CLIPAway
    & 0.193 & 0.925 & 0.323 & 0.666 & 21.95 \\
  PowerPaint
    & 0.275 & 0.934 & 0.278 & 0.715 & 23.87 \\
  SmartEraser
    & \best{0.106} & 0.939 & 0.257 & \best{0.734} & 25.36 \\
  FLUX.2 (native)
    & 0.359 & 0.766 & 0.537 & 0.525 & 23.42 \\
  \midrule
  \ours{} (FLUX.2)
    & 0.166 & \best{0.943} & \best{0.253} & 0.649 & \best{29.57} \\
  \bottomrule
  \end{tabular}%
  }
\end{table}

\paragraph{Implementation details.}
Unless noted, \ours{} freezes FLUX.2-klein-4B~\cite{flux2024,esser2024sd3} ($T{=}14$ flow steps, CFG $w_{\mathrm{cfg}}{=}3.5$) and I-JEPA ViT-H/14-1K~\cite{assran2023ijepa}; neither Fill nor JEPA weights are updated, and there is no paired removal training.
I-JEPA sees a gray-filled instance $I_{\mathrm{vis}}$ (Eq.~\eqref{eq:grayfill}); $\mathbf{E}_{\mathrm{target}}$ is cached once per image.
Editable support uses dilation $r{=}4$ with contact extents $\sigma{=}\max(6,0.5h_{\mathrm{obj}})$ and $\delta_x{=}\max(8,0.35w_{\mathrm{obj}})$ (Eqs.~\eqref{eq:wedge}--\eqref{eq:mflux}).
Full \ours{} source-prefills from $\mathrm{Encode}(I\odot(1-M_{\mathrm{flux}}))$ and uses a shared shadow-aware erasure prompt rather than instance-specific captions (Technical Appendix: Prompts).
Packed gating $P$ is a nearest-neighbor downsample of $M_{\mathrm{flux}}$ onto the Fill latent grid, so guidance, pinning, and $\mathrm{FMStep}$ address the same coordinates.
Images use longest side $768$\,px at inference and are scored at $1024{\times}1024$; alignment uses $224{\times}224$ decoded previews.
Sparse guidance applies $\eta{=}0.45$ at $\mathcal{T}_{\mathrm{guide}}{=}\{4,2\}$ ($t_{\mathrm{end}}{=}4$, $n{=}2$) with projected locking (Eqs.~\eqref{eq:grad}--\eqref{eq:pin}).
These defaults (Technical Appendix Table~1) are held fixed across reported splits.
Local scores average five seeds $\mathcal{S}{=}\{22,\ldots,26\}$ with base seed $22$.
On a single NVIDIA A100 40\,GB GPU, Full \ours{} runs in ${\sim}5.3$--$6.3\,s$/image versus ${\sim}3.0\,s$ for native FLUX.2 under the same timing protocol.

\subsection{Main Results}
\label{sec:main-results}

Tables~\ref{tab:main} and~\ref{tab:defacto} report results under the OmniEraser and SmartEraser protocols, respectively.

\paragraph{RemovalBench and RORD-Val.}
On RemovalBench, \ours{} (FLUX.2) reduces native FLUX.2 CMMD from 0.496 to 0.108 and raises PSNR from 22.70 to 24.36\,dB.
These are the axes most sensitive to leftover object identity and support discoloration under the OmniEraser protocol.
Supervised OmniEraser remains stronger on FID and LPIPS (39.52 vs.\ 52.69 FID).
On RORD-Val, \ours{} improves native FID from 149.02 to 55.59 and CMMD from 0.644 to 0.305, with the best listed LPIPS and PSNR among the compared methods; OmniEraser still leads FID and CMMD.
Figure~\ref{fig:qual} shows outside-mask cleanup on three RemovalBench scenes.

\paragraph{DEFACTO-Val.}
On DEFACTO-Val (Table~\ref{tab:defacto}), \ours{} attains the best listed ReMOVE, LPIPS, and PSNR among the compared methods, while supervised SmartEraser remains strongest on CMMD and SSIM.
Every reported axis improves over native FLUX.2.

\subsection{Qualitative Results}

Figure~\ref{fig:qual} visualizes the Table~\ref{tab:main} gap on three RemovalBench pairs under the same instance-only masks.
The comparison isolates two failure modes that full-image metrics are designed to catch: incomplete erasure under a frozen Fill backbone, and residual object-linked effects that survive even when the instance hole looks plausible.

Native FLUX.2 struggles on every row.
On the banana desk (row~1), it leaves a dark smear where the fruit sat; on the tiled phone (row~2), a ghostly hole breaks the grout pattern instead of a clean completion; on the desk mouse (row~3), a flat gray silhouette remains with residual contact shading.
The object footprint may shrink, but the supporting plane still reads as edited.

Supervised OmniEraser~\cite{wei2025omnieraser} is consistently cleaner inside $M_{\mathrm{obj}}$, yet outside-mask effects still leak on the support.
It leaves faint discoloration on the desk and a dirty patch across the tile line: residuals that instance-only masks do not cover but that $M_{\mathrm{flux}}$ exposes to Fill.

\ours{} rewrites the object together with nearby contact residuals in one guided pass when those residuals fall inside $M_{\mathrm{flux}}$.
Desk tone and grout continue more coherently than under native Fill, which under-edits outside $M_{\mathrm{obj}}$.
OmniEraser, trained for this protocol, is often cleaner inside the instance. Remaining support stains show that instance-only masks still leave an effect-region problem, which $M_{\mathrm{flux}}$ exposes to Fill.

\subsection{Ablation Study}

We use two RemovalBench ablations under the Table~\ref{tab:main} protocol.
One disables one module at a time under the benchmark-aligned support; the other swaps only the frozen guidance prior with $M_{\mathrm{flux}}$, source prefill, shadow-aware prompting, and schedule fixed.
Full protocols, RORD-Val repeats, and multi-seed testing are in the Technical Appendix.

\paragraph{Module ablation.}
Table~\ref{tab:ablation} removes shadow-aware prompting, source prefill, or JEPA alignment from an otherwise Full stack.
Pure FLUX.2 establishes the native reference point (FID 113.92; CMMD 0.496).
The w/o Shadow Prompt and w/o Prefill variants improve on this reference but retain substantial FID/CMMD gaps to Full.
w/o JEPA is the strict expanded-support text-only control: it preserves $M_{\mathrm{flux}}$, source prefill, and the shadow-aware prompt but removes all feature guidance.
Its FID 59.30, CMMD 0.292, and PSNR 23.81\,dB indicate that spatial exposure and text conditioning alone do not account for Full's CMMD of 0.108.
These matched removals isolate I-JEPA under a fixed $M_{\mathrm{flux}}$; they do not compare contact-band geometry with isotropic dilation or $M_{\mathrm{obj}}$-only edits.
Full is best on every reported axis; the same ordering holds on RORD-Val (see Table~\ref{tab:ablation-rord}), with qualitative examples in Fig.~\ref{fig:ablation-qual}.

\begin{table}[!t]
  \centering
  \footnotesize
  \setlength{\tabcolsep}{3.4pt}
  \setlength{\extrarowheight}{1.2pt}
  \renewcommand{\arraystretch}{1.05}
  \caption{Ablations on RemovalBench ($n{=}69$) under the same local protocol as Table~\ref{tab:main} (FID/CMMD/LPIPS/PSNR/AS).
  Full matches \ours{} (FLUX.2), Pure FLUX.2 matches native FLUX.2, and w/o JEPA is the expanded-support text-only control.
  \textbf{Bold} = best per column.}
  \label{tab:ablation}
  \begin{tabular}{l@{\hspace{0.35em}} rrrrr}
  \toprule
  \textbf{Variant}
    & FID\down & CMMD\down & LPIPS\down & PSNR\up & AS\up \\
  \midrule
  Pure FLUX.2
    & 113.92 & 0.496 & 0.184 & 22.70 & 4.62 \\
  w/o Shadow Prompt
    & 77.21 & 0.380 & 0.181 & 23.20 & 4.65 \\
  w/o Prefill
    & 80.82 & 0.328 & 0.179 & 23.30 & 4.69 \\
  w/o JEPA
    & 59.30 & 0.292 & 0.177 & 23.81 & 4.70 \\
  \midrule
  Full \ours{}
    & \best{52.69} & \best{0.108} & \best{0.175} & \best{24.36} & \best{4.74} \\
  \bottomrule
  \end{tabular}
\end{table}

\begin{table}[!t]
  \centering
  \footnotesize
  \setlength{\tabcolsep}{3.4pt}
  \setlength{\extrarowheight}{1.2pt}
  \renewcommand{\arraystretch}{1.05}
  \caption{Guidance prior comparison on RemovalBench ($n{=}69$): only the frozen guidance backbone changes; $M_{\mathrm{flux}}$, prefill, shadow-aware prompt, and sparse schedule are fixed (Technical Appendix: Guidance prior comparison).
  w/ JEPA matches Table~\ref{tab:ablation} Full row; CLIP / DINOv2 patch rows use the same $\eta$ and $\mathcal{T}_{\mathrm{guide}}$.
  \textbf{Bold} = best per column among prior swaps.}
  \label{tab:ablation-prior}
  \begin{tabular}{l@{\hspace{0.35em}} rrrrr}
  \toprule
  \textbf{Prior}
    & FID\down & CMMD\down & LPIPS\down & PSNR\up & AS\up \\
  \midrule
  w/ CLIP patch
    & 107.60 & 0.150 & 0.235 & 23.81 & 4.63 \\
  w/ DINOv2 patch
    & 110.00 & 0.148 & 0.243 & 23.71 & 4.60 \\
  \midrule
  w/ JEPA
    & \best{52.69} & \best{0.108} & \best{0.175} & \best{24.36} & \best{4.74} \\
  \bottomrule
  \end{tabular}
\end{table}

\paragraph{Guidance prior comparison.}
Table~\ref{tab:ablation-prior} isolates whether the gains depend on context-conditioned prediction rather than frozen feature guidance in general.
w/ CLIP patch and w/ DINOv2 patch replace I-JEPA with patch-level semantic and neighborhood-matching energies plugged into the same gated update (Eqs.~\ref{eq:grad}--\ref{eq:pin}); details are in Technical Appendix: Guidance prior comparison.
Under these matched settings, both alternatives trail w/ JEPA on every reported axis. For example, CMMD is $0.150/0.148$ versus $0.108$, and FID is $107.6/110.0$ versus $52.69$.
All priors share the same $\eta$ and schedule. Loss scales differ across backbones, so the swap is an operating-point comparison rather than a fully retuned one.
CLIP patch is slightly stronger than DINOv2 patch on FID, LPIPS, and PSNR; neither matches I-JEPA at this operating point (Fig.~\ref{fig:prior-qual}).
The result is consistent with using I-JEPA as a hole-prediction prior rather than as a generic perceptual regularizer.

\subsection{Limitations and Future Work}

\ours{} is less reliable when the target occupies most of the frame: the hole prior is conditioned on visible context, so a near full-image mask leaves too little evidence (Fig.~\ref{fig:failure}).
The contact band assumes upright ground contact; side lighting, detached shadows, and mirror or water reflections fall outside this geometry.
As with other training-free editors, appearance inherits Fill-backbone biases.
A dependence-based support, for example I-JEPA or Fill-score residuals on unmasked patches, would be a natural extension when residuals detach from the contact contour.

\section{Conclusion}

We framed training-free object-and-effect removal as two coupled questions: which pixels may change, and what structure should occupy the instance hole.
A contact-band expansion exposes local residuals on the supporting plane to frozen Fill; I-JEPA supplies a representation-space hole target; sparse projected updates align the decoded completion inside the instance.
On standard clean-plate protocols the method improves a frozen FLUX.2 backbone under instance-only masks.
{\small
\bibliographystyle{ieeenat_fullname}
\bibliography{aaai2026}
}

\clearpage
\appendix
\section*{Supplementary Material}
\addcontentsline{toc}{section}{Supplementary Material}
\input{technical-appendix}
\end{document}

%% file: preamble.tex
\usepackage{cuted} %< for strip environment

\usepackage{currfile} %< for \currfilebase macro

\usepackage{caption} %< for "\captionof" commands (teaser)

%% file: technical-appendix.tex
% Technical appendix body, input by EW.tex after the main bibliography.
% Not a standalone document.
\section{Methodology Details}
\label{app:method}

This supplementary material defines the operators used in the Methodology, motivates the I-JEPA guidance prior, and records local properties of Eqs.~9--10.
The analysis establishes exact background preservation and one-step descent under a standard smoothness condition.
It does not imply global convergence, metric improvement, or superiority over alternative guidance priors; those questions are addressed empirically in Tables~3 and~4.

\subsection{Definitions}
\label{app:defs}

\begin{itemize}
    \item $p_{\mathrm{Fill}}(\cdot\mid I,M_{\mathrm{flux}})$: conditional law of frozen flow-matching Fill on the gated support.
    \item $q_{\mathrm{pred}}(I_{\mathrm{vis}},\mathcal{I}_{\mathrm{vis}},\mathcal{I}_{\mathrm{mask}})$: the deterministic masked-token prediction implemented by frozen I-JEPA.
    \item $\mathbf{E}_{\mathrm{target}}=q_{\mathrm{pred}}(I_{\mathrm{vis}},\mathcal{I}_{\mathrm{vis}},\mathcal{I}_{\mathrm{mask}})$: the one-shot cached hole-token target.
    \item $\mathcal{Z}\subseteq\mathbb{R}^{d}$: packed Fill latent space; $P\in\{0,1\}^{d}$: packed gate of $M_{\mathrm{flux}}$ (see Implementation details).
\end{itemize}

The latent-alignment energy on decoded hole patches is
\begin{equation}
\begin{aligned}
\mathcal{L}_{\mathrm{align}}(\mathbf{z};\,\mathbf{E}_{\mathrm{target}})
&=
\frac{1}{|\mathcal{I}_{\mathrm{mask}}|}
\sum_{i\in\mathcal{I}_{\mathrm{mask}}}
\|\mathbf{u}_i-\mathbf{v}_i\|_2^2,\\
\mathbf{u}_i &= \phi_{\mathrm{JEPA}}\bigl(\mathrm{Decode}(\mathbf{z})\bigr)_i,\\
\mathbf{v}_i &= \mathbf{E}_{\mathrm{target},i}.
\end{aligned}
\label{eq:app-align}
\end{equation}
Both $\mathbf{u}_i$ and $\mathbf{v}_i$ live in the frozen I-JEPA ViT-H/14 representation, matching Eq.~8.

Given a gate $P$, define the edit and background subspaces
\begin{align}
\mathcal{E}_P
&:=
\bigl\{\mathbf{z}\in\mathcal{Z}:\ (1-P)\odot\mathbf{z}=\mathbf{0}\bigr\},
\label{eq:app-edit}\\
\mathcal{B}_P
&:=
\bigl\{\mathbf{z}\in\mathcal{Z}:\ P\odot\mathbf{z}=\mathbf{0}\bigr\}.
\label{eq:app-bg}
\end{align}
Every $\mathbf{z}\in\mathcal{Z}$ admits the unique Euclidean split $\mathbf{z}=P\odot\mathbf{z}+(1-P)\odot\mathbf{z}$ with summands in $\mathcal{E}_P$ and $\mathcal{B}_P$.
Guidance may change only the $\mathcal{E}_P$ component.

For a cached pin $\mathbf{z}^{\mathrm{pin}}$ and proposal $\mathbf{u}$, the gated restoration map is
\begin{equation}
\Pi_P(\mathbf{z}^{\mathrm{pin}};\,\mathbf{u})
:=
P\odot\mathbf{u}
+(1-P)\odot\mathbf{z}^{\mathrm{pin}}.
\label{eq:app-pi}
\end{equation}
Geometrically, $\Pi_P(\mathbf{z}^{\mathrm{pin}};\,\cdot)$ is the affine projector onto the coset $\mathbf{z}^{\mathrm{pin}}+\mathcal{E}_P$.

\subsection{Design rationale and analysis scope}
\label{app:scope}

\paragraph{I-JEPA guidance prior.}
Object removal under a fixed mask admits a masked-completion view: visible context constrains what may plausibly replace the hole.
Because I-JEPA predicts masked patch representations from visible tokens, its frozen predictor provides a context-conditioned target $\mathbf{E}_{\mathrm{target}}$ without a removal-specific head.
The inductive bias is \emph{scene extension from context}: hole tokens should be predictable from surrounding structure.
Prop.~\ref{prop:task-align} formalizes why this interface matches I-JEPA pretraining more directly than CLIP or DINO priors; empirical dominance over CLIP/DINO guidance is not claimed without direct baselines.

\paragraph{Task alignment versus CLIP/DINO priors.}
Object removal under gray-filled $I_{\mathrm{vis}}$ is a \emph{masked patch prediction} problem: indices $\mathcal{I}_{\mathrm{mask}}$ should be explained by visible context $\mathcal{I}_{\mathrm{vis}}$.
Let $\hat{\mathbf{v}}_i$ denote the $i$-th hole token predicted from gray-filled context:
\begin{equation}
\begin{split}
\hat{\mathbf{v}}
&=
\psi_{\mathrm{JEPA}}\!\bigl(
\phi_{\mathrm{JEPA}}(I_{\mathrm{vis}}),
\mathcal{I}_{\mathrm{vis}},
\mathcal{I}_{\mathrm{mask}}
\bigr),\\
\mathcal{L}_{\mathrm{JEPA}}
&=
\frac{1}{|\mathcal{I}_{\mathrm{mask}}|}
\sum_{i\in\mathcal{I}_{\mathrm{mask}}}
\bigl\|
\hat{\mathbf{v}}_i
-
\mathrm{sg}\bigl[\phi_{\mathrm{JEPA}}(I)_i\bigr]
\bigr\|_2^2.
\end{split}
\label{eq:app-jepa-pretrain}
\end{equation}
where $\mathrm{sg}[\cdot]$ stops gradients through the EMA target encoder.
At inference, $\mathbf{E}_{\mathrm{target}}=\hat{\mathbf{v}}$ in Eq.~\ref{eq:app-align}.
The alignment objective reuses I-JEPA's patchwise prediction-target geometry, but evaluates the current decoded completion $\hat{I}_t=\mathrm{Decode}(\mathbf{z})$ against the cached prediction rather than against an EMA target encoding of a clean image.

Training-free CLIP and DINO priors optimize structurally different energies.
A CLIP-style term
\begin{equation}
\mathcal{L}_{\mathrm{CLIP}}(\mathbf{z})
=
-\cos\!\Bigl(
\mathrm{CLIP}_{\mathrm{img}}(\mathrm{Decode}(\mathbf{z})),\,
\mathrm{CLIP}_{\mathrm{text}}(t)
\Bigr)
\label{eq:app-clip}
\end{equation}
conditions on global image--text semantics rather than patch-wise continuation from $\mathcal{I}_{\mathrm{vis}}$; it defines no mask-indexed target computable from $I_{\mathrm{vis}}$ alone.
A DINO-style patch term
\begin{equation}
\mathcal{L}_{\mathrm{DINO}}(\mathbf{z})
=
\frac{1}{|\mathcal{I}_{\mathrm{mask}}|}
\sum_{i\in\mathcal{I}_{\mathrm{mask}}}
\bigl\|
\phi_{\mathrm{DINO}}(\mathrm{Decode}(\mathbf{z}))_i
-
\mathbf{r}_i
\bigr\|_2^2
\label{eq:app-dino}
\end{equation}
requires reference tokens $\mathbf{r}_i$ on the hole.
Without the clean plate, $\mathbf{r}_i$ must be approximated heuristically (neighbor copying, averaging, or iterative self-targeting), reintroducing the ambiguity removal seeks to resolve.
Neither alternative directly produces a mask-indexed, context-only prediction for the missing region under the interfaces considered here.

\begin{proposition}[Context-only completion target]
\label{prop:task-align}
For fixed $(I_{\mathrm{vis}},\mathcal{I}_{\mathrm{vis}},\mathcal{I}_{\mathrm{mask}})$, the cached target
$\mathbf{E}_{\mathrm{target}}=\hat{\mathbf{v}}$ depends only on visible context and mask indices.
Under Eqs.~\ref{eq:app-align}--\ref{eq:app-jepa-pretrain}, $\mathcal{L}_{\mathrm{align}}$ inherits I-JEPA's masked prediction-target geometry.
In contrast, $\mathcal{L}_{\mathrm{CLIP}}$ lacks mask-indexed context targets, and $\mathcal{L}_{\mathrm{DINO}}$ lacks a clean-plate-free reference $\mathbf{r}_i$ on $\mathcal{I}_{\mathrm{mask}}$.
Among these three frozen guidance priors, only I-JEPA implements context-conditioned hole prediction at inference without ground-truth plates.
\end{proposition}
\begin{remark}
Prop.~\ref{prop:task-align} is a \emph{task-alignment} statement, not an empirical dominance claim: it explains why I-JEPA supplies a well-defined completion prior for removal, while CLIP/DINO require auxiliary targets or global semantics.
Whether I-JEPA guidance outperforms CLIP/DINO guidance on RemovalBench is tested under matched \ours{} settings in Table~4.
\end{remark}

\paragraph{Scope of the formal results.}
Props.~\ref{prop:invar}--\ref{prop:descent} below are \emph{operator checks} for Eqs.~9--10.
They explain why test-time guidance does not drift the unmasked image and why each guided micro-step decreases $\mathcal{L}_{\mathrm{align}}$ on editable coordinates under an $L$-smoothness assumption.
They do \emph{not} establish optimality of the latent-alignment objective, convergence of the full flow trajectory, or a link from feature alignment to FID/CMMD/PSNR.

\paragraph{Gray-fill support.}
Gray-filling exactly $M_{\mathrm{obj}}$, together with excluding its patch indices from the visible set, removes object identity from the predictor input while retaining surrounding context.
Gray-filling the larger set $M_{\mathrm{flux}}$ would discard contact-surface tokens that Fill still needs; leaving object pixels would leak the instance into the guidance prior.

\paragraph{Sparse guidance schedule.}
Decoded previews must be sufficiently formed for meaningful feature comparison, while guidance at every transition can unnecessarily constrain appearance.
The sparse set $\mathcal{T}_{\mathrm{guide}}=\{4,2\}$ for $T{=}14$ therefore applies two late corrective updates and leaves the remaining Fill transitions native.
Hard locking (Eq.~10) further ensures guided steps cannot corrupt the visible coordinates that both models condition on.

\subsection{Effect-aware support}
\label{app:support}

If the editable mask is only $M_{\mathrm{obj}}$, the sampler is constrained to preserve cast effects outside that support.
The contact-aware band $M_{\mathrm{shadow}}$ extends the edit region toward the supporting plane; after dilation,
\begin{equation}
\mathrm{supp}(M_{\mathrm{obj}})
\subseteq
\mathrm{supp}(M_{\mathrm{flux}})
\subseteq
\mathrm{dilate}_r\bigl(\mathrm{supp}(M_{\mathrm{obj}})\cup W\bigr),
\label{eq:app-nest}
\end{equation}
where $W$ is the residual band from the Methodology.
Mapping masks to packed gates yields $P_{\mathrm{obj}}\le P_{\mathrm{flux}}$ entrywise and hence $\mathcal{E}_{P_{\mathrm{obj}}}\subseteq\mathcal{E}_{P_{\mathrm{flux}}}$.

\begin{proposition}[Enlarged edit support]
\label{prop:nest}
If $\mathbf{z}\in\mathbf{z}^{\mathrm{pin}}+\mathcal{E}_{P_{\mathrm{obj}}}$, then $\mathbf{z}\in\mathbf{z}^{\mathrm{pin}}+\mathcal{E}_{P_{\mathrm{flux}}}$.
Whenever $P_{\mathrm{flux}}$ contains additional active entries, it strictly enlarges the set of coordinates Fill may revise, a necessary condition for outside-mask cleanup.
This is a feasibility statement about the edit region, not a guarantee of clean-plate quality.
\end{proposition}
\begin{proof}
From $P_{\mathrm{obj}}\le P_{\mathrm{flux}}$, the constraint $(1-P_{\mathrm{obj}})\odot(\mathbf{z}-\mathbf{z}^{\mathrm{pin}})=\mathbf{0}$ implies $(1-P_{\mathrm{flux}})\odot(\mathbf{z}-\mathbf{z}^{\mathrm{pin}})=\mathbf{0}$.
\end{proof}

\subsection{Gated update properties}
\label{app:gated}

In the method, the proposal at a guided step is $\mathbf{u}=\mathbf{z}^{\mathrm{pin}}-\eta G$ with $G=\nabla_{\mathbf{z}}\mathcal{L}_{\mathrm{align}}$, followed by $\Pi_P$ (Eqs.~9--10).
The next results formalize background preservation and one-step descent of the alignment energy; they certify the update operator, not end-to-end removal quality.

\begin{proposition}[Background preservation]
\label{prop:invar}
For every $\mathbf{u}\in\mathcal{Z}$,
\[
(1-P)\odot\Pi_P(\mathbf{z}^{\mathrm{pin}};\,\mathbf{u})=(1-P)\odot\mathbf{z}^{\mathrm{pin}}.
\]
Hence every coordinate with $P_i=0$ is unchanged by guidance.
\end{proposition}
\begin{proof}
Expand Eq.~\ref{eq:app-pi} and use $P\odot(1-P)=\mathbf{0}$ entrywise.
\end{proof}

Assume $\mathcal{L}_{\mathrm{align}}$ is $L$-smooth on $\mathcal{Z}$ after freezing Decode and JEPA.
One guided micro-step is
\begin{equation}
\mathbf{z}^+
=
\Pi_P\bigl(\mathbf{z};\,\mathbf{z}-\eta G(\mathbf{z})\bigr),
\qquad
G(\mathbf{z})=\nabla_{\mathbf{z}}\mathcal{L}_{\mathrm{align}}(\mathbf{z}).
\label{eq:app-step}
\end{equation}

\begin{proposition}[One-step alignment descent]
\label{prop:descent}
If $0<\eta\le 1/L$, then
\[
\mathcal{L}_{\mathrm{align}}(\mathbf{z}^+)
\le
\mathcal{L}_{\mathrm{align}}(\mathbf{z})
-\frac{\eta}{2}\,\bigl\|P\odot G(\mathbf{z})\bigr\|_2^2.
\]
Moreover $\mathbf{z}^+$ and $\mathbf{z}$ share the same background coordinates (Prop.~\ref{prop:invar}).
\end{proposition}
\begin{proof}
Because $\mathbf{z}$ already satisfies the background constraint, Eq.~\ref{eq:app-step} gives $\mathbf{z}^+-\mathbf{z}=-\eta(P\odot G(\mathbf{z}))$.
$L$-smoothness then yields
$\mathcal{L}_{\mathrm{align}}(\mathbf{z}^+)\le
\mathcal{L}_{\mathrm{align}}(\mathbf{z})
-\eta\|P\odot G(\mathbf{z})\|_2^2
+\tfrac{L\eta^2}{2}\|P\odot G(\mathbf{z})\|_2^2$.
The result follows from $\eta\le 1/L$.
\end{proof}

\begin{remark}[Scope]
Props.~\ref{prop:invar}--\ref{prop:descent} are local, operator-level facts about Eqs.~9--10.
They do not assert global optimality along the flow.
Prop.~\ref{prop:task-align} motivates the I-JEPA guidance prior structurally; Table~3 tests its empirical contribution within \ours{}.
Prop.~\ref{prop:invar} rules out background drift, while Prop.~\ref{prop:descent} establishes only local decrease of the chosen feature objective; neither guarantees improved clean-plate metrics.
\end{remark}

\subsection{Stationary prior design}
\label{app:bilevel}

Precomputing $\mathbf{E}_{\mathrm{target}}$ once defines the following fixed-target surrogate over guided states:
\begin{equation}
\begin{aligned}
\min_{\{\mathbf{z}_t\}_{t\in\mathcal{T}_{\mathrm{guide}}}}
&\sum_{t\in\mathcal{T}_{\mathrm{guide}}}
\mathcal{L}_{\mathrm{align}}(\mathbf{z}_t;\,\mathbf{E}_{\mathrm{target}}) \\
&\text{s.t.}\quad
\mathbf{z}_t\in\mathbf{z}_t^{\mathrm{pin}}+\mathcal{E}_{P_t},
\end{aligned}
\label{eq:app-bilevel}
\end{equation}
where each pin $\mathbf{z}_t^{\mathrm{pin}}$ is produced by the preceding frozen FM transition.
The sampler does not jointly optimize this surrogate; rather, each guided step performs one projected update on its corresponding term.
Keeping $\mathbf{E}_{\mathrm{target}}$ fixed prevents a moving-target objective in which the prior follows the current completion.
This stationarity clarifies the update, but does not establish that I-JEPA is the optimal feature space for clean-plate agreement.

\section{Experimental Details}
\label{app:experiments}

This supplementary material specifies the protocol for every locally scored row, including implementation settings, multi-seed aggregation, statistical tests, metric definitions, and failure cases.

\subsection{Implementation details}
\label{app:impl}

\paragraph{Prompts and text conditioning.}
\label{app:prompts}
Both Fill runs use a short erasure prompt rather than instance-specific captions.
The default \textbf{shadow-aware} positive prompt used by Full \ours{} is
\begin{quote}
\ttfamily\small
Clean empty background, seamless inpainting, natural lighting,\\
no object, no person, no cast shadow, no contact shading, no text,\\
photorealistic.
\end{quote}
with negative prompt
\begin{quote}
\ttfamily\small
object, person, animal, text, watermark, logo, blurry, low quality,\\
extra limbs, distorted, silhouette, floating debris, shadow residual.
\end{quote}
The \textbf{w/o Shadow Prompt} ablation replaces this with a shorter generic prompt
\begin{quote}
\ttfamily\small
Clean empty background, seamless inpainting, photorealistic.
\end{quote}
(and drops ``shadow residual'' from the negative list), keeping all other modules fixed.
Classifier-free guidance is $w_{\mathrm{cfg}}{=}3.5$ for FLUX.2-klein-4B.
Neither $M_{\mathrm{flux}}$ nor JEPA alignment modifies the text embedding: the former changes only spatial conditioning, and the latter updates only latents selected by $P_t$.

\paragraph{Source prefill.}
In Full, Fill is initialized from $\mathrm{Encode}(I\odot(1-M_{\mathrm{flux}}))$ (Alg.~1), preserving visible context outside the editable support.
The \textbf{w/o Prefill} ablation gray-fills $M_{\mathrm{flux}}$ before encoding, removing those source cues while keeping $M_{\mathrm{flux}}$, the shadow-aware prompt, and JEPA guidance unchanged.

\paragraph{Image and latent geometry.}
Images are processed with longest side 768 pixels and resized to $1024{\times}1024$ for benchmark evaluation.
The VAE spatial stride is approximately 8, with 16-pixel alignment; latents are stored in FLUX packed format.

\paragraph{Packed-latent gating.}
$P_t$ is constructed by nearest-neighbor downsampling of $M_{\mathrm{flux}}$ to the VAE grid and packing the resulting cells into FLUX's latent-token layout.
Thus gating, pinning, and $\mathrm{FMStep}$ address identical latent coordinates.
JEPA alignment uses a separate pixel mask at decode resolution (nearest resize of $M_{\mathrm{obj}}$).

\paragraph{Hyperparameters.}
Table~\ref{tab:hyperparams} lists the default \ours{} settings used for all locally scored rows unless noted otherwise.
Fill and I-JEPA weights remain frozen; scale adaptivity in $M_{\mathrm{shadow}}$ comes from per-instance $(h_{\mathrm{obj}}, w_{\mathrm{obj}})$ (Eq.~5).

\paragraph{Hyperparameter selection.}
\label{app:hyper-select}
During development we swept a small grid on a held-out visual check set (not used for reported tables) and selected the Table~\ref{tab:hyperparams} operating point by clean-plate CMMD/PSNR on RemovalBench together with qualitative residual cleanup:
$\eta\in\{0.25,0.35,0.45,0.55\}$,
$t_{\mathrm{end}}\in\{6,4,2\}$ with $n\in\{1,2,3\}$ guided steps among late timesteps,
dilation $r\in\{2,4,6\}$,
FLUX CFG $w_{\mathrm{cfg}}\in\{2.5,3.5,5.0\}$,
and shadow coefficients in Eq.~5 within $\{\,0.4,0.5,0.6\,\}$ (vertical) and $\{\,0.25,0.35,0.45\,\}$ (horizontal).
The reported defaults are the selected setting; we did not retune per benchmark split after freezing this configuration.

\begin{table}[!t]
  \centering
  \scriptsize
  \setlength{\tabcolsep}{3pt}
  \renewcommand{\arraystretch}{1.1}
  \caption{Default \ours{} hyperparameters.}
  \label{tab:hyperparams}
  \begin{tabular}{@{}llll@{}}
  \toprule
  \textbf{Setting} & \textbf{Symbol} & \textbf{Value} & \textbf{Notes} \\
  \midrule
  Fill steps (FLUX.2-klein-4B) & $T$ & 14 & Frozen sampler \\
  Guided timesteps & $\mathcal{T}_{\mathrm{guide}}$ & $\{4,2\}$ & $t_{\mathrm{end}}{=}4$, $n{=}2$ \\
  Guidance step size & $\eta$ & 0.45 & Eq.~9 \\
  Mask dilation radius & $r$ & 4 & Eq.~6 \\
  Shadow extent (vertical) & $\sigma$ & $\max(6,0.5h_{\mathrm{obj}})$ & Eq.~5 \\
  Shadow extent (horizontal) & $\delta_x$ & $\max(8,0.35w_{\mathrm{obj}})$ & Eq.~5 \\
  CFG (FLUX.2-klein-4B) & $w_{\mathrm{cfg}}$ & 3.5 & Prompts \\
  Input longest side & --- & 768 px & Eval at $1024{\times}1024$ \\
  JEPA decode preview & --- & $224{\times}224$ & For $\mathcal{L}_{\mathrm{align}}$ \\
  Multi-seed base / set & --- & 22 / 22--26 & Multi-seed tests \\
  \bottomrule
  \end{tabular}
\end{table}

\paragraph{Runtime and compute infrastructure.}
\label{app:compute}
All locally scored experiments and timing runs use a single NVIDIA A100-40GB GPU (40\,GB HBM) on a Linux host with 512\,GB system RAM.
Software stack: Ubuntu 22.04, CUDA 12.1, Python 3.10, PyTorch 2.2, Diffusers-compatible FLUX.2-klein-4B loaders, and the public I-JEPA ViT-H/14 checkpoint.
Metric scripts follow the released OmniEraser and SmartEraser evaluation code (FID/CMMD/LPIPS/PSNR/AS or ReMOVE/SSIM as applicable).
We measure single-image wall-clock latency after two warm-up runs, excluding model loading and metric computation.
Runs use longest side 768/640 with $T{=}14$ flow steps.
\ours{} on FLUX.2-klein-4B averages ${\sim}5.3\,s$ under the default config used for RORD-Val and ${\sim}5.7$--$6.3\,s$ under the stronger JEPA config used for RemovalBench Full in Table~1.
End-to-end inference remains within single-digit seconds on an A100, with no paired removal training or parameter updates.
Native FLUX.2 without \ours{} averages ${\sim}3.0\,s$ per image under the same timing protocol.

\paragraph{Preprocessing.}
\label{app:preprocess}
Benchmark images and instance masks are taken from the public RemovalBench, RORD-Val, and DEFACTO-Val releases cited in the main paper.
Our preprocessing is limited to: (i) reading RGB images and binary masks; (ii) resizing with longest side 768 for inference; (iii) resizing outputs to $1024{\times}1024$ for protocol-aligned scoring; (iv) gray-filling $M_{\mathrm{obj}}$ for the I-JEPA branch and constructing $M_{\mathrm{flux}}$ as in the Methodology.
No additional learning-based preprocess or proprietary filters are used.
Inference code is available at \url{https://github.com/xiuwk0820-collab/PredErase}, subject to upstream Fill / I-JEPA checkpoint licenses.

\subsection{Benchmark splits}
\label{app:benchmarks}

\paragraph{RemovalBench.}
Following~\cite{wei2025omnieraser}, we evaluate on $69$ valid $1024{\times}1024$ real object--clean-plate pairs with instance-only masks.
Metrics are full-image FID, CMMD, LPIPS (SqueezeNet), PSNR, and Aesthetic Score (AS).
Because cast shadows and contact shading often lie outside $M_{\mathrm{obj}}$, this split is our primary stress test for effect-aware removal under clean-plate ground truth.

\paragraph{RORD-Val.}
RORD-Val applies the same OmniEraser metrics to a larger clean-plate validation set ($n{=}343$), covering broader scene scales and content variation.

\paragraph{DEFACTO-Val.}
DEFACTO-Val follows the SmartEraser protocol on synthetic removals at $1024{\times}1024$, scored with the public metric code on CMMD, ReMOVE (SAM ViT-H), AlexNet LPIPS, SSIM, and PSNR.
It provides a second community protocol beyond the OmniEraser real clean-plate setting.

\subsection{Metrics and evaluation protocol}
\label{app:metrics}

\paragraph{Protocol alignment.}
\label{app:fairness}
\textbf{Protocol-aligned} denotes a shared split, masks, ground truth, and metric implementation with OmniEraser (RemovalBench and RORD-Val) or SmartEraser (DEFACTO-Val); it does not imply retraining those methods on FLUX.2-klein-4B.
On RemovalBench and RORD-Val, Table~1 follows the OmniEraser Table~1 layout at $1024{\times}1024$ and reports classical and diffusion inpainters (ZITS++, LaMa, BrushNet, FLUX.1-Fill), training-free editors (CLIPAway, Attentive Eraser), supervised OmniEraser, native FLUX.2, and \ours{}, all scored in our environment under the same full-image metrics.
Recent supervised removers such as ObjectClear~\cite{zhao2026objectclear} are discussed in Related Work but omitted from Table~1.
Ablations in Table~3 reuse the RemovalBench scoring of Table~1; Table~\ref{tab:ablation-rord} reports the same FLUX.2-klein-4B variants on RORD-Val under the identical protocol.
On DEFACTO-Val, Table~2 follows the SmartEraser baseline layout (ZITS++, LaMa, RePaint, SD-Inpaint, CLIPAway, PowerPaint, SmartEraser, native FLUX.2, and \ours{}) on CMMD/ReMOVE/LPIPS/SSIM/PSNR via the public metric code.
For per-image LPIPS/PSNR on RemovalBench ($n{=}69$), the paired Wilcoxon test reported below confirms that Full \ours{} improves over native FLUX.2 at $p{\ll}0.05$.

\paragraph{Metric definitions.}
\label{app:metric-defs}
Unless noted, scores compare a method output $\hat{I}$ to the clean plate $I^{\mathrm{gt}}$ at $1024{\times}1024$ and are averaged over the test set; AS is reference-free.
\begin{itemize}
    \item \textbf{PSNR} (primary with CMMD): full-image pixel fidelity~\cite{wang2004ssim}; residual shadows outside $M_{\mathrm{obj}}$ lower the score.
    \item \textbf{CMMD}~\cite{jayasumana2024cmmd,radford2021clip} (primary): CLIP-space distributional discrepancy between predictions and clean plates; reported as in OmniEraser. Sensitive to global appearance, including lighting residuals.
    \item \textbf{LPIPS}~\cite{zhang2018lpips}: SqueezeNet features under the OmniEraser suite (RemovalBench and RORD-Val, full image) and AlexNet features under SmartEraser; the two settings are not numerically interchangeable.
    \item \textbf{AS}: LAION aesthetic predictor mean over outputs (no GT)~\cite{schuhmann2022laion}. Preference rather than clean-plate fidelity.
    \item \textbf{FID}~\cite{heusel2017fid}: Inception feature distance under the OmniEraser suite (RemovalBench and RORD-Val).
    \item \textbf{SSIM} / \textbf{ReMOVE} (DEFACTO-Val)~\cite{wang2004ssim}: structural similarity and SmartEraser's SAM-based removal completeness score, both via the public SmartEraser scripts.
\end{itemize}
We omit mask-restricted DINO distances from Table~1 because they use a different feature and normalization pipeline from our DINOv2 patch distances; placing them in one column would not constitute a valid comparison.
Primary claims use full-image FID/CMMD/LPIPS/PSNR/AS.
Image-level scores are means over the split and then across seeds; FID and CMMD summarize set-level distributions per seed before seed averaging.

\subsection{Ablation details}
\label{app:ablation-details}

\begin{figure*}[!t]
    \centering
    \setlength{\tabcolsep}{0pt}%
    {\ewlab
    \begin{tabular}{@{}*{5}{>{\centering\arraybackslash}p{\dimexpr\textwidth/5\relax}@{}}}
    Input & Pure FLUX.2 & w/o JEPA & Full (Ours) & GT \\[-0.15em]
    \multicolumn{5}{@{}c@{}}{\includegraphics[width=\textwidth]{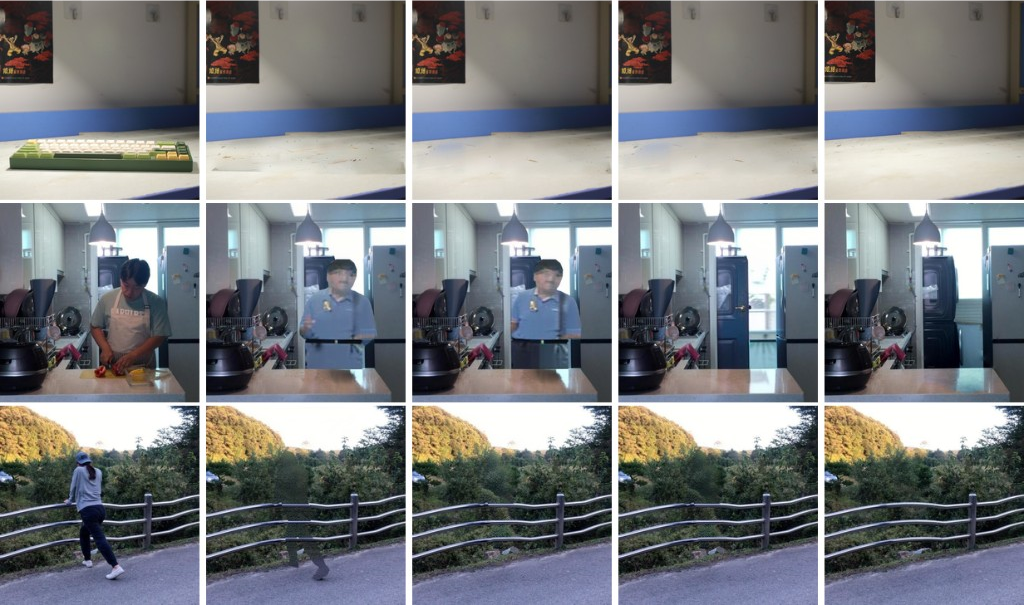}}
    \end{tabular}}
    \caption{Qualitative module ablation under the same instance-only masks.
    Columns: Input, Pure FLUX.2 (native), w/o JEPA (expanded $M_{\mathrm{flux}}$ without feature guidance), Full \ours{}, and clean-plate GT.
    Pure FLUX.2 leaves gray fills or ghost silhouettes; w/o JEPA improves cleanup but remains structurally incomplete; Full recovers support and background structure closer to GT.}
    \label{fig:ablation-qual}
\end{figure*}

Figure~\ref{fig:ablation-qual} visualizes the module ladder corresponding to Table~3.
On desk, kitchen, and outdoor scenes, Pure FLUX.2 leaves flat fills or ghost silhouettes inside the hole.
The w/o JEPA control (expanded $M_{\mathrm{flux}}$ with text conditioning only) reduces some residuals but still fails to reconstruct coherent support or background structure.
Full \ours{} removes the target and completes the revealed region closer to the clean plate, consistent with the CMMD/PSNR gap between w/o JEPA and Full in Table~3.

\paragraph{Guidance prior comparison (CLIP / DINOv2 / I-JEPA).}
\label{app:ablation-prior}

Table~4 compares frozen guidance priors on RemovalBench ($n{=}69$) while holding the remainder of Full \ours{} fixed:
$M_{\mathrm{flux}}$ from Eq.~6, source prefill (Alg.~1), the shadow-aware erasure prompt specified above, $T{=}14$, $\mathcal{T}_{\mathrm{guide}}{=}\{4,2\}$, $\eta{=}0.45$, projected locking to the current Fill state outside $P$ (Eqs.~9--10), and five-seed aggregation (below).
Only the test-time feature objective changes.

\textbf{Shared protocol.}
At each $t\in\mathcal{T}_{\mathrm{guide}}$, we decode $\hat{I}_t=\mathrm{Decode}(\mathbf{z}_t)$, evaluate a prior-specific $\mathcal{L}_{\mathrm{prior}}$ on hole patches, backpropagate through the frozen decoder only, and apply the same gated step and pin as Full \ours{}.

\textbf{w/ CLIP patch.}
We use OpenAI CLIP ViT-L/14 at $224{\times}224$~\cite{radford2021clip} with patch-level guidance on $\mathcal{I}_{\mathrm{mask}}$ under the shadow-aware erasure prompt specified above.

\textbf{w/ DINOv2 patch.}
We use a frozen DINOv2 ViT-L/14 backbone~\cite{oquab2024dinov2} and patch tokens on $\mathcal{I}_{\mathrm{mask}}$.
Hole references $\mathbf{r}_i$ are formed from a neighbor average of visible patch tokens (Eq.~\ref{eq:app-dino}): each hole index $i\in\mathcal{I}_{\mathrm{mask}}$ takes the mean embedding of its $k{=}8$ nearest visible-patch neighbors in $\phi_{\mathrm{DINO}}(\hat{I}_t)$.

\textbf{w/ JEPA.}
The I-JEPA row is Full \ours{}: $\mathbf{E}_{\mathrm{target}}$ is computed once from gray-filled $I_{\mathrm{vis}}$, and $\mathcal{L}_{\mathrm{align}}$ is evaluated only on $\mathcal{I}_{\mathrm{mask}}$ (Eqs.~8,~\ref{eq:app-jepa-pretrain}).
Seed-averaged RemovalBench scores for the prior swap are FID $107.6/110.0$ (CLIP / DINOv2 patch) versus $52.69$ (w/ JEPA), and CMMD $0.150/0.148$ versus $0.108$.

\begin{figure*}[!t]
    \centering
    \setlength{\tabcolsep}{0pt}%
    {\ewlab
    \begin{tabular}{@{}*{5}{>{\centering\arraybackslash}p{\dimexpr\textwidth/5\relax}@{}}}
    Input & w/ CLIP & w/ DINOv2 & w/ JEPA (Ours) & GT \\[-0.15em]
    \multicolumn{5}{@{}c@{}}{\includegraphics[width=\textwidth]{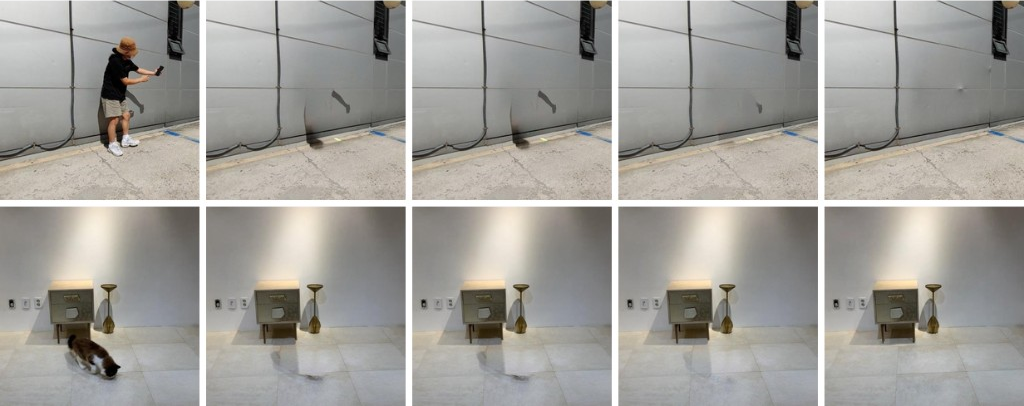}}
    \end{tabular}}
    \caption{Qualitative guidance prior comparison under the matched Full stack (same $M_{\mathrm{flux}}$, prefill, shadow-aware prompt, and sparse schedule as Table~4).
    Columns: Input, w/ CLIP patch, w/ DINOv2 patch, w/ JEPA (Full \ours{}), and clean-plate GT.
    CLIP and DINOv2 remove the instance but leave dark support-plane residuals; JEPA cleans cast/contact shading closer to GT.}
    \label{fig:prior-qual}
\end{figure*}

Figure~\ref{fig:prior-qual} shows the same ordering visually: under identical edit support and schedule, CLIP/DINOv2 guidance leaves cast-shadow smudges on the wall/ground and floor tiles, whereas JEPA recovers a cleaner support closer to the clean plate.

\paragraph{RORD-Val ablations.}
\label{app:ablation-rord}

Table~3 in the main paper isolates each \ours{} component on RemovalBench ($n{=}69$); Table~4 compares frozen guidance priors under the matched Full stack.
Table~\ref{tab:ablation-rord} repeats the five module variants on RORD-Val ($n{=}343$) under the OmniEraser full-image metric suite used in Table~1.
Pure FLUX.2 and Full \ours{} match the native and \ours{} (FLUX.2) rows in Table~1; intermediate rows disable one module at a time with all other settings fixed.
On this larger split, removing shadow-aware prompting or source prefill still leaves large FID/CMMD gaps to Full, and disabling JEPA guidance degrades every clean-plate axis relative to Full while remaining ahead of Pure FLUX.2 on FID/CMMD/LPIPS/PSNR.
Native Fill retains the highest AS (5.08), consistent with Table~1; we therefore treat FID/CMMD/LPIPS/PSNR as the primary ablation axes on RORD-Val.

\begin{table}[t]
  \centering
  \footnotesize
  \setlength{\tabcolsep}{3.2pt}
  \renewcommand{\arraystretch}{1.15}
  \caption{Ablations on RORD-Val ($n{=}343$) under the same local protocol as Table~1 (FID/CMMD/LPIPS/PSNR/AS).
  Full matches \ours{} (FLUX.2) and Pure FLUX.2 matches native FLUX.2 in Table~1.
  \textbf{Bold} = best per column on clean-plate axes (FID/CMMD/LPIPS/PSNR); AS is reference-free.}
  \label{tab:ablation-rord}
  \begin{tabular}{@{}lccccc@{}}
  \toprule
  \textbf{Variant}
    & FID\down & CMMD\down & LPIPS\down & PSNR\up & AS\up \\
  \midrule
  Pure FLUX.2
    & 149.02 & 0.644 & 0.168 & 18.49 & \best{5.08} \\
  w/o Shadow Prompt
    & 93.0 & 1.33 & 0.150 & 20.0 & 5.02 \\
  w/o Prefill
    & 98.5 & 1.14 & 0.138 & 20.3 & 4.98 \\
  w/o JEPA
    & 65.7 & 1.00 & 0.126 & 21.8 & 4.90 \\
  Full \ours{}
    & \best{55.59} & \best{0.305} & \best{0.114} & \best{23.45} & 4.84 \\
  \bottomrule
  \end{tabular}
\end{table}

\paragraph{Multi-seed runs and statistical testing.}
\label{app:seed}

Fill sampling is stochastic; a single random draw can therefore overstate or understate performance.
For all locally scored \ours{} and native FLUX.2 configurations reported in Tables~1,~3,~4, and~\ref{tab:ablation-rord}, we therefore repeat full-split inference over multiple independent random seeds.

Unless stated otherwise, we use five seeds $\mathcal{S}=\{22,23,24,25,26\}$, with 22 as the default seed.
Each seed $s\in\mathcal{S}$ defines a complete pass over the evaluation split under identical checkpoints, prompts, and hyperparameters.
For per-image metrics (LPIPS, PSNR, AS, and DEFACTO SSIM/ReMOVE), let $m_i^{(s)}$ denote the score of image $i$ under seed $s$.
We first form the seed-averaged image score $\bar{m}_i=\frac{1}{|\mathcal{S}|}\sum_{s\in\mathcal{S}} m_i^{(s)}$, then report the split mean $\frac{1}{N}\sum_{i=1}^{N}\bar{m}_i$ (e.g., RemovalBench $N{=}69$).
For set-level metrics (FID, CMMD), we compute one score per seed on the full prediction set and report the mean over $\mathcal{S}$.
For Full \ours{}, the seed-to-seed standard deviations of RemovalBench split means are $0.002$ for LPIPS and $0.07\,$dB for PSNR, indicating stable aggregate performance across these runs.
Tables report these seed-averaged point estimates; we do not bold variance bands in the main tables to match the OmniEraser / SmartEraser reporting layout.

We compare Full \ours{} with native FLUX.2 using a paired two-sided Wilcoxon signed-rank test on seed-averaged per-image scores $\{\bar{m}_i\}_{i=1}^{N}$ (SciPy \texttt{wilcoxon}; asymptotic approximation; zero differences omitted).
The null is that the median paired difference is zero; we reject at $\alpha{=}0.05$.
On RemovalBench ($N{=}69$), Full \ours{} versus native FLUX.2 yields
\begin{itemize}
    \item LPIPS: median paired improvement $0.009$, $W{=}382$, $p{=}8.0{\times}10^{-7}$;
    \item PSNR: median paired improvement $1.65\,$dB, $W{=}291$, $p{=}4.3{\times}10^{-8}$.
\end{itemize}
Both tests reject the null at $p{\ll}0.05$.
We apply the same paired protocol to ablations against Full (Table~3).
All four contrasts (w/o JEPA, w/o Prefill, w/o Shadow Prompt, Pure FLUX.2) are significant on both LPIPS and PSNR at $\alpha{=}0.05$, and remain so under Bonferroni correction over the four contrasts ($\alpha{=}0.0125$); the weakest is Full vs.\ w/o JEPA (LPIPS $p{=}9.4{\times}10^{-3}$, PSNR $p{=}3.1{\times}10^{-3}$).
FID/CMMD are set-level scores and are not paired at the image level; for those axes we rely on seed-averaged table entries.
Paired tests are reported for native FLUX.2 and the matched ablations; remaining baseline comparisons use split-level metrics.

\subsection{Failure cases on large targets}
\label{app:failure}

I-JEPA constructs $\mathbf{E}_{\mathrm{target}}$ exclusively from context outside $M_{\mathrm{obj}}$.
When the mask covers most of the image---for example, a near full-frame person or large foreground furniture---too little visible evidence remains to constrain the hole.
Guided Fill may then produce structural artifacts, texture drift, or incomplete effect cleanup even when $M_{\mathrm{flux}}$ covers the relevant region.
Representative outputs are shown in Fig.~\ref{fig:failure}.

\IfFileExists{figures/failure_cases.png}{%
\begin{figure}[!t]
    \centering
    \setlength{\tabcolsep}{0pt}%
    {\ewlab
    \begin{tabular}{@{}*{3}{>{\centering\arraybackslash}p{\dimexpr\linewidth/3\relax}@{}}}
    Input & Output & GT \\[-0.15em]
    \multicolumn{3}{@{}c@{}}{\includegraphics[width=\linewidth]{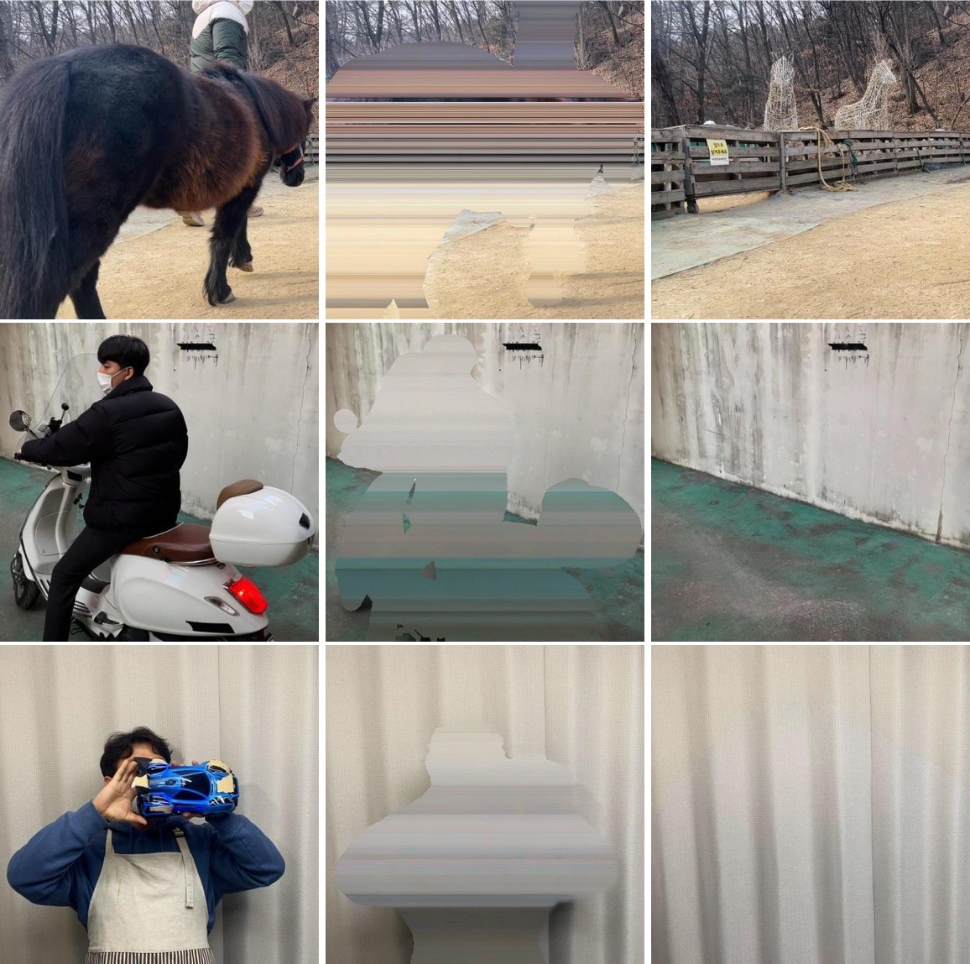}}
    \end{tabular}}
    \caption{Failure cases on large targets (RORD-Val).
    Columns: Input, \ours{} output, clean-plate GT.
    When the instance dominates the frame, visible context is thin and JEPA-guided Fill under-constrains the hole, yielding streaky or smeared silhouettes rather than coherent support reconstruction.}
    \label{fig:failure}
\end{figure}
}{}

\section{Limitations and Future Work}
\label{app:limitations}

\subsection{Limitations}

\paragraph{Benchmark-aligned geometry.}
RemovalBench and RORD-Val evaluate upright real scenes with instance-only masks and clean-plate ground truth; our contact-band $M_{\mathrm{shadow}}$ is sized from the lowest object rows under that convention (Eq.~5).
The construction targets cast shadows and contact shading on supporting surfaces, the dominant outside-mask residuals in these benchmarks; it is not a general physical model for reflections, floating objects, or arbitrary illumination effects.
Equation~4 assumes the downward image axis is the contact normal; when orientation metadata are available, the operator can be rotated accordingly.
Without such metadata, \ours{} does not infer non-horizontal supports.

\paragraph{Large-target context.}
When $M_{\mathrm{obj}}$ covers most of the frame, the visible complement is thin and JEPA-guided Fill can under-constrain the hole, producing streaky or incomplete completions (Fig.~\ref{fig:failure}).
This pattern is most visible on large foreground instances in RORD-Val, and the current method provides no dedicated fallback for this regime.

\paragraph{Training-free operating point.}
Fill and I-JEPA remain frozen; all adaptation occurs through test-time latent updates.
Supervised baselines consequently retain advantages on several reported axes, while \ours{} targets training-free clean-plate fidelity under instance-only masks.

\paragraph{Formal analysis scope.}
The Methodology Details above record implementation checks for hard locking and one-step alignment descent, together with task-alignment rationale for I-JEPA versus CLIP/DINO guidance priors (Prop.~\ref{prop:task-align}); Table~4 supplies the matched empirical comparison.
Operator properties (Props.~\ref{prop:invar}--\ref{prop:descent}) hold by proof under the stated smoothness and gating assumptions.
Design claims that are empirical in nature (module contributions and guidance prior choice) are validated on RemovalBench and RORD-Val in Tables~3,~4 and~\ref{tab:ablation-rord}.
Together these cover the paper's theoretical statements: formal propositions by proof, and the associated empirical consequences by matched experiments.

\subsection{Future Work}

The factorized design permits replacing Eq.~4 with an off-the-shelf instance-shadow detector while retaining the same JEPA-guided Fill stage.
Evaluating that extension requires a detector whose output and training data are compatible with the instance-only benchmarks.
Other priorities include hierarchical or multi-stage synthesis for large targets, broader frozen-backbone evaluation, video removal, and richer lighting models that combine predictive representations with temporal or physical cues.